\pdfoutput=1

\documentclass[11pt]{article}

\usepackage{emnlp2021}

\usepackage{times}
\usepackage{latexsym}

\usepackage[T1]{fontenc}

\usepackage[utf8]{inputenc}

\usepackage{microtype}

\usepackage{hyperref}
\usepackage{url}
\usepackage[utf8]{inputenc} 
\usepackage[T1]{fontenc}    
\usepackage{hyperref}       
\usepackage{url}            
\usepackage{booktabs}       
\usepackage{amsfonts}       
\usepackage{nicefrac}       
\usepackage{microtype}      
\newcommand{\regdsl}[1]{\texttt{\small #1}}
\usepackage{booktabs}
\usepackage{tabularx}
\usepackage{todonotes}
\usepackage{multirow}
\usepackage{url}
\usepackage{amsmath}
\usepackage{mathtools}
\usepackage{amssymb}
\usepackage{amsthm}
\usepackage{tablefootnote}
\usepackage{caption}
\usepackage{subcaption}
\usepackage{bbm}
\usepackage{comment}
\usepackage{xspace}
\usepackage{algorithm}
\usepackage{proof}
\usepackage{stmaryrd}
\newtheorem{theorem}{Theorem}
\usepackage{bm}
\usepackage{pifont}
\usepackage[noend]{algpseudocode}
\usepackage{wrapfig}

\theoremstyle{definition}

\newcommand{\worklist}{\mathcal{Q}}
\newcommand{\cmark}{\ding{51}}%
\newcommand{\xmark}{\ding{55}}%
\newtheorem{definition}{Definition}[section]
\newcommand{\irule}[2]{\mkern-2mu\displaystyle\frac{#1}{\vphantom{,}#2}\mkern-2mu}

\DeclareMathOperator*{\argmax}{arg\,max}

\newcommand{\streg}{\textsc{StructuredRegex}}

\newcommand{\partialprog}{P}

\newcommand{\nl}{{{N}}}
\newcommand{\spec}{{{\phi}}}
\newcommand{\dsl}{{{L}}}
\newcommand{\model}{{{M}}}
\newcommand{\grammar}{\mathcal{G}}
\newcommand{\examples}{\mathcal{E}}
\newcommand{\pose}{\mathcal{E^+}}
\newcommand{\nege}{\mathcal{E^-}}

\algnewcommand\Input{\textbf{input: }}
\algnewcommand\Output{\textbf{output: }}

\newcommand{\toolname}{\textsc{OpSynth}\xspace}

\title{Optimal Neural Program Synthesis from Multimodal Specifications}

\author{Xi Ye\quad\quad Qiaochu Chen\quad\quad Isil Dillig\quad\quad Greg Durrett\\
  Department of Computer Science \\
  The University of Texas at Austin \\
  \texttt{\{xiye,qchen,isil,gdurrett\}@cs.utexas.edu} \\ }

\begin{document}
\maketitle
\begin{abstract}
Multimodal program synthesis, which leverages different types of user input to synthesize a desired program, is an attractive way to scale program synthesis to challenging settings; however, it requires integrating noisy signals from the user, like natural language, with hard constraints on the program's behavior. This paper proposes an \emph{optimal neural synthesis} approach where the goal is to find a program that satisfies user-provided constraints while also maximizing the program's score with respect to a neural model. Specifically, we focus on multimodal synthesis tasks in which the user intent is expressed using a combination of natural language (NL) and input-output examples. At the core of our method is a top-down recurrent neural model that places distributions over abstract syntax trees conditioned on the NL input. This model not only allows for efficient search over the space of syntactically valid programs, but it allows us to leverage \emph{automated program analysis} techniques for pruning the search space based on infeasibility of partial programs with respect to the user's constraints. The experimental results on a multimodal synthesis dataset ({\sc StructuredRegex}) show that our method substantially outperforms prior state-of-the-art techniques in terms of accuracy and efficiency, and finds model-optimal programs more frequently.\footnote{Code available: \href{https://github.com/xiye17/OpSynth}{https://github.com/xiye17/OpSynth}}
\end{abstract}

\section{Introduction}

In recent years, there has been a revolution in machine learning-based program synthesis techniques for automatically generating programs from high-level expressions of user intent, such as input-output examples~\citep{deepcoder,song1,robustfill,repl,polozov1,song2} and natural language~\citep{sqlizer,dong16,asn,tranx, gulwaninl, robsql}. Many of these techniques use deep neural networks to consume user input and then perform model-guided search to find a program
that satisfies the user. However, because 
both natural language and input examples 
can be inherently ambiguous~\citep{robustfill,conala}, a recent thread of work on \emph{multimodal synthesis} attempts to combine different types of cues 
to allow program synthesis to effectively scale to more complex problems. Critically, this setting introduces a new challenge: how do we efficiently synthesize programs with a combination of hard and soft constraints from distinct sources?

\begin{figure}[t]
\centering
\includegraphics[width=\linewidth,trim=450 150 450 150,clip]{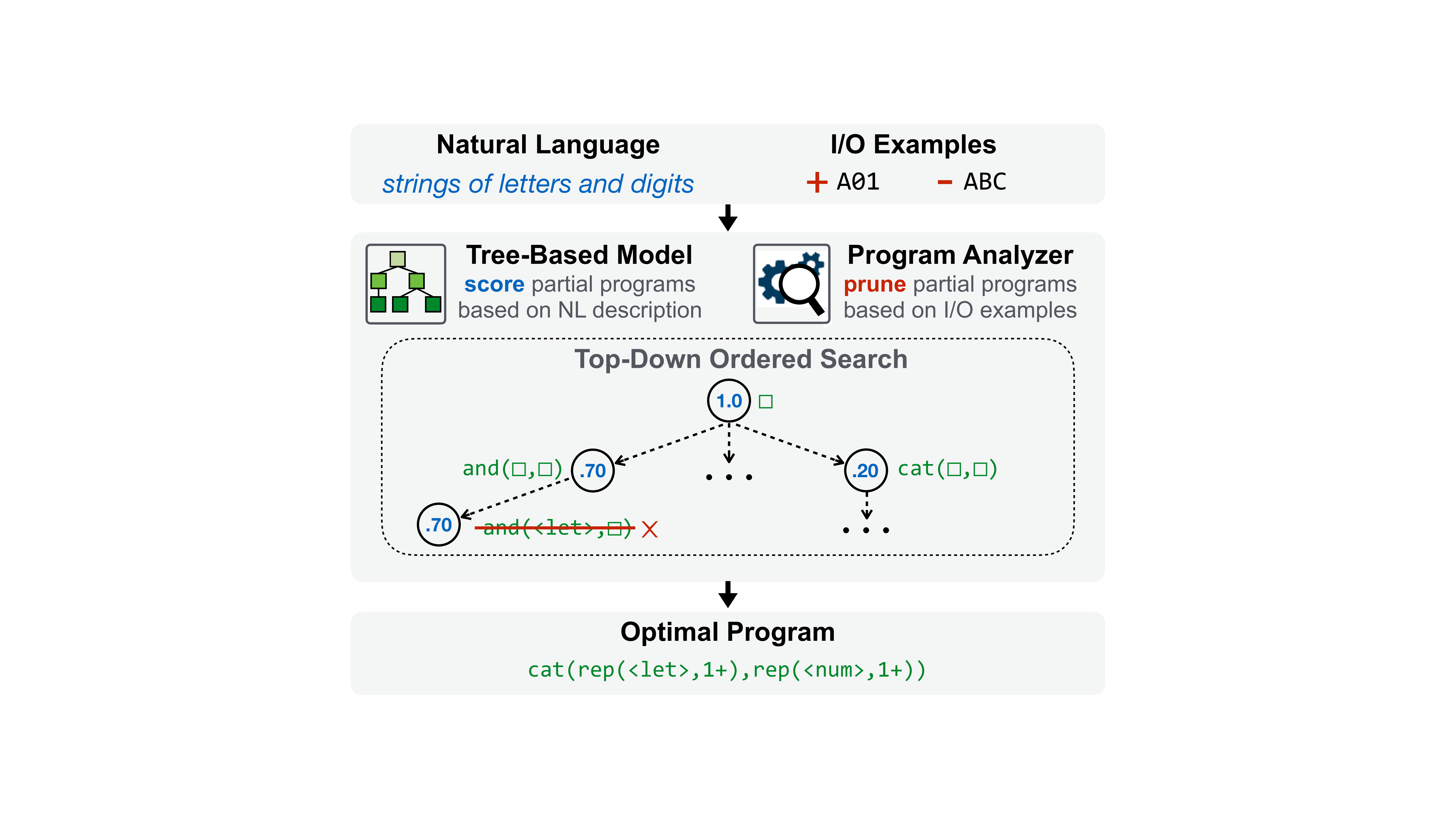}
\caption{The framework of our multi-modal optimal synthesis approach. A tree-structured model scores partial programs based on the NL description and a program analyzer prunes the search space based on the I/O examples. Our algorithm searches in a best-first fashion following the scores, and hence ensures the optimality of the output program with respect to the model. }
\label{fig:framework}
\vspace{-0.1in}
\end{figure}

The core contribution of this paper is to formulate multimodal synthesis as an \emph{optimal synthesis} task and propose an optimal synthesis algorithm to solve it.
The goal of optimal synthesis is to generate a program that satisfies any hard constraints provided by the user while also maximizing the score under a learned neural network model that captures noisy information, like that from natural language.
In practice, there are \emph{many} programs that satisfy the hard constraints, so this maximization is crucial to finding the user's intended program: if our neural model is well-calibrated, a program that maximizes the score under the neural model is more likely to be what the user wants.

In our setting (Figure~\ref{fig:framework}), we train a neural model to take natural language input that can be used to guide the search for a program consistent with user-provided examples. Because our search procedure enumerates programs according to their score (values in blue in Figure~\ref{fig:framework}), the first enumerated program satisfying the examples is guaranteed to be optimal according to the model. A central feature of our approach is the use of a tree-structured neural model, namely the \emph{abstract syntax network (ASN)}~\citep{asn}, for constructing syntactically valid programs in a top-down manner. The structure of the ASN model restricts search to programs that are syntactically correct, thereby avoiding the need to deal with program syntax errors~\citep{spoc}, and it allows us to search over programs in a flexible way, without constraining a left-to-right generation order like seq2seq models do. More importantly, the use of top-down search allows us to more effectively leverage \emph{automated program analysis} techniques for proving infeasibility of partial ASTs. As a result, our synthesizer can prune the search space more aggressively than prior work and significantly speed up search. While our network structure and pruning techique are adapted from prior work, we combine them and generalize them to this optimal neural synthesis setting in a new way, and we show that our general approach leads to substantial improvements over previous synthesis methods.

We  implement our method in a synthesizer called \toolname
and evaluate it on the challenging {\sc StructuredRegex} dataset~\citep{structregex} for synthesizing regular expressions from linguistically diverse natural language descriptions and positive/negative examples. We compare our approach against a range of techniques from prior work and ablations of our own method.  \toolname\ achieves substantial gain over past work by solving 60.8\% (resp. 48.8\%) of the programs of Test (resp. Test-E) set in \streg{}. These results represent a roughly 7-10\% improvement over prior work with a roughly $3
\times$ speedup due to the improved pruning.

\section{Problem Formulation}\label{sec:formulation}

\paragraph{Context-free grammar.} In this work, we assume that the syntax of the target programming language $\dsl$ is specified as a context-free grammar  $\grammar = (V, \Sigma, R, S_0)$ where $V$ is a set of non-terminals, $\Sigma$ is the set of terminal symbols, $R$ is a set of productions, and $S_0$ is the start symbol. We use the notation $s$ to denote any \emph{symbol} in $V \cup \Sigma$. The grammar in Figure~\ref{fig:grammar} has two nonterminals ($S_0$ and $V_1$) and three terminals (\texttt{\small cat}, \texttt{\small <0>}, and \texttt{\small <1>}). To simplify presentation in the rest of the paper,  we assume that each grammar production is of the form $v \rightarrow f(s_0, \ldots, s_n)$ where $f$ is a language construct (e.g., a constant like $0$ or a built-in function/operator like {\tt cat}, $+$, etc.). 

We represent programs in terms of their abstract syntax trees (AST). We assume that every node $n$ in the tree is labeled with a grammar symbol $s$ (denoted $\mathcal{S}(n)$) and a production $r$ (denoted $\mathcal{R}(n)$) that indicates the CFG production that was used to assign the terminal symbol for that node (if applicable). Figure \ref{fig:ast_p} shows an AST representation of the program $\texttt{cat}(\texttt{cat}(\texttt{<0>},\texttt{<1>}),\texttt{<0>})$ generated using the simple grammar shown in Figure \ref{fig:grammar}. Similar AST representations have been used in recent work on grammar-based program generation models \cite{tranx, asn, sun2020treegen}.

\begin{figure*}[t]
\vspace{-0.2in}
    \centering
    \begin{minipage}{0.22\textwidth}
    \[
    \begin{array}{ll}
    S_0 \rightarrow& V_1 \\
    V_1 \rightarrow& \texttt{<0>} \\
    & | \texttt{<1>}\\
    & | \texttt{cat}(V_1, V_1)
    \end{array}
    \]
    \caption{Example grammar for a simple language.}
    \label{fig:grammar}
    \end{minipage}
    \hspace{0.1in}
    \begin{minipage}{0.36\textwidth}
    \begin{center}
        \includegraphics[width=1.0\textwidth,trim=250 200 85 160,clip]{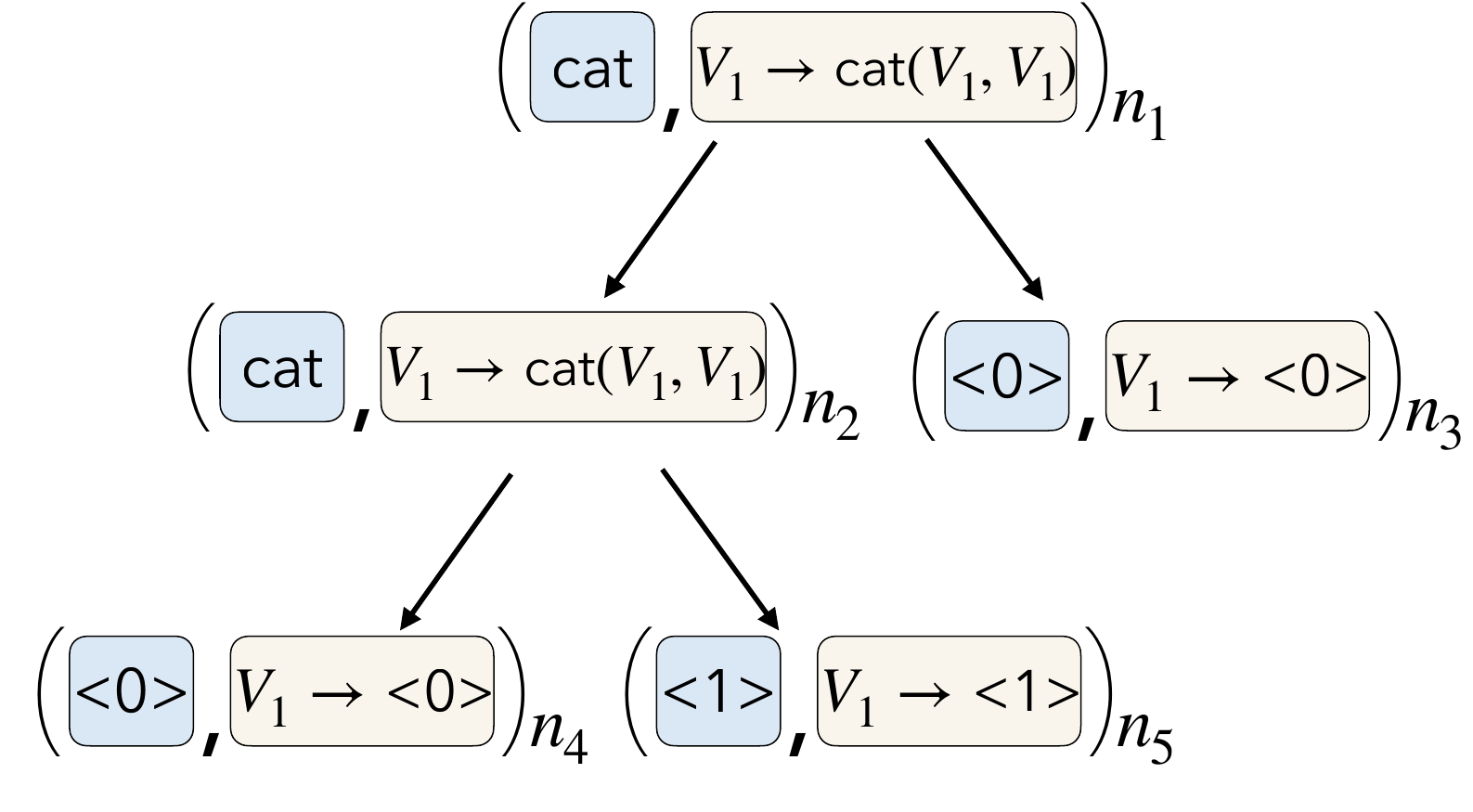}
        \caption{Example of an AST derivation of \texttt{\small{cat(cat(<0>,<1>),<0>)}}. Blue boxes represent symbols and yellow boxes represent productions.  } \label{fig:ast_p}
    \end{center}
    \end{minipage}
    \hspace{0.1in}
    \begin{minipage}{0.36\textwidth}
    \begin{center}
        \includegraphics[width=1.0\textwidth,trim=230 190 100 180,clip]{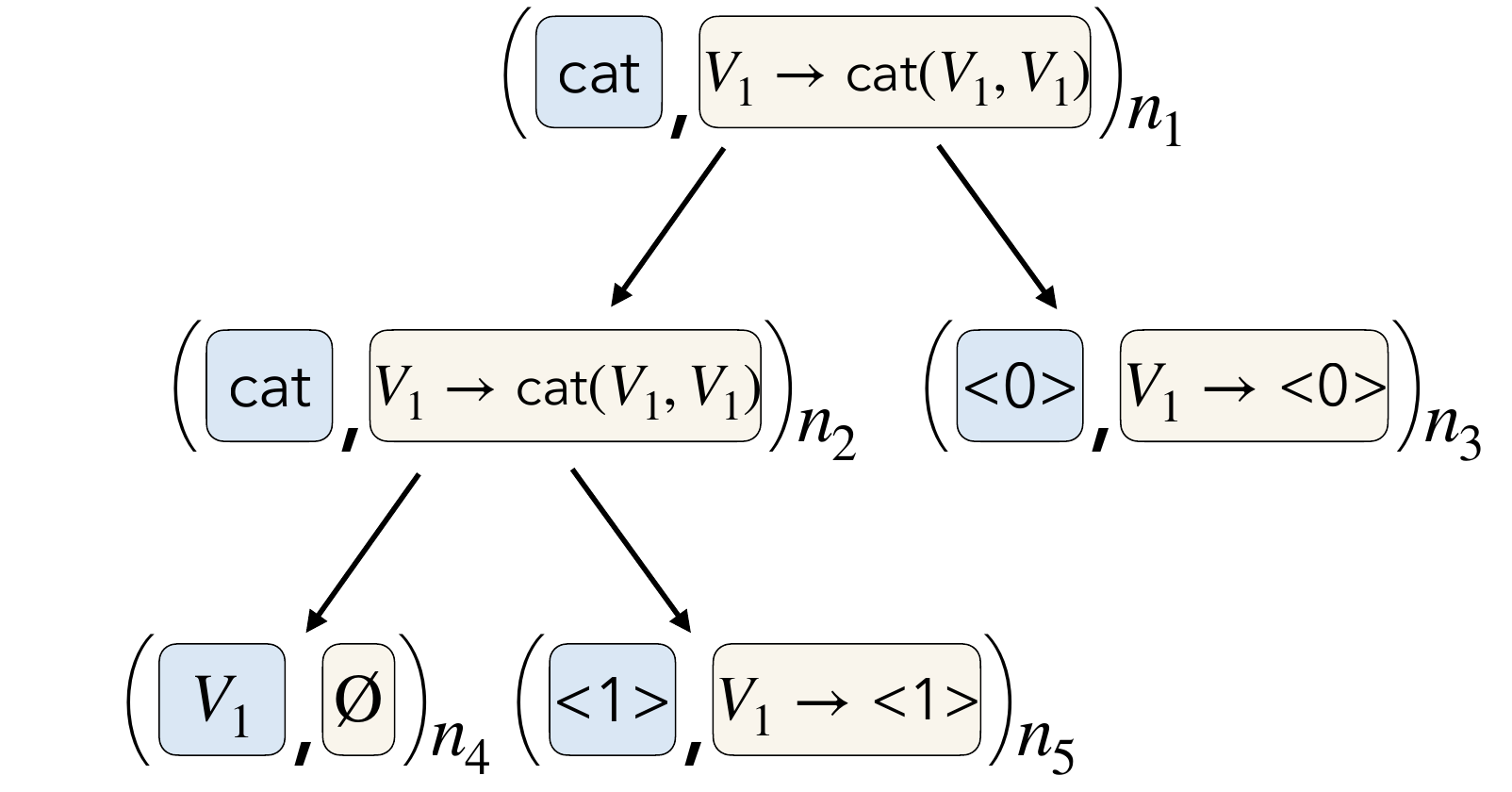}
        \caption{Example of a partial program. $n_4$ is a leaf node with non-terminal symbol $V_1$. }\label{fig:ast_pp}
    \end{center}
    \end{minipage}
\end{figure*}
 
\paragraph{Partial programs.} For the purposes of this paper, a \emph{partial program} is an AST in which some of the nodes are labeled with non-terminal symbols in the grammar (see Figure~\ref{fig:ast_pp}). For a \emph{complete program}, all node labels are terminal symbols. We use the notation {\sc Expand}$(\partialprog, l, r)$ to denote replacing leaf $l$ with production $r$, which adds $n$ nodes $s_1, \ldots, s_n$ to the tree corresponding to the yield of $r$.

\paragraph{Consistency with examples.} In this paper, we focus on the multimodal synthesis problem where the user provides a logical specification $\phi$ and  a natural 
language description. Specifically, we focus on logical specifications in the form of positive and negative examples of the program behavior. Each example is a pair $(x, y)$ such that, for a positive example, we have $P(x)=y$ for the target program $P$, and for a negative example, we have $P(x) \neq y$.  Given a set of examples $\examples = \pose \cup \nege$ and program $P$, we write $P \models \examples$, if we have $P(x) = y$ for every positive example in $\pose$ and we have $P(x) \neq y$ for every negative example in $\nege$. If $P$ is a partial program,  $P \not \models \phi$ indicates that there is no completion of $P$ that  satisfies the specification $\phi$.

\paragraph{Optimal multimodal synthesis problem.} A second input to our multimodal synthesis problem is a natural language description of the task. We define a model $\model_{\theta}(P \mid \nl)$ that yields the probability of a given program conditioned on the description (Section~\ref{sec:algorithm}). Given a programming language $\dsl$ specified by its context-free grammar, a logical specification $\spec$ (e.g., a set of positive and negative examples), natural language description $\nl$, and a model $\model_{\theta}$, our goal is to find the most likely program in the language satisfying the constraints:
\begin{align}\label{eq:opt-synth}
 \argmax_{ \partialprog \in \dsl \ \land \ P \models \spec} \model_{\theta}(\partialprog \ | \ \nl)   
\end{align}

\section{Optimal Neural Synthesis Algorithm}

\begin{figure*}[t]
\centering
\includegraphics[width=\textwidth,trim=0 200 0 200,clip]{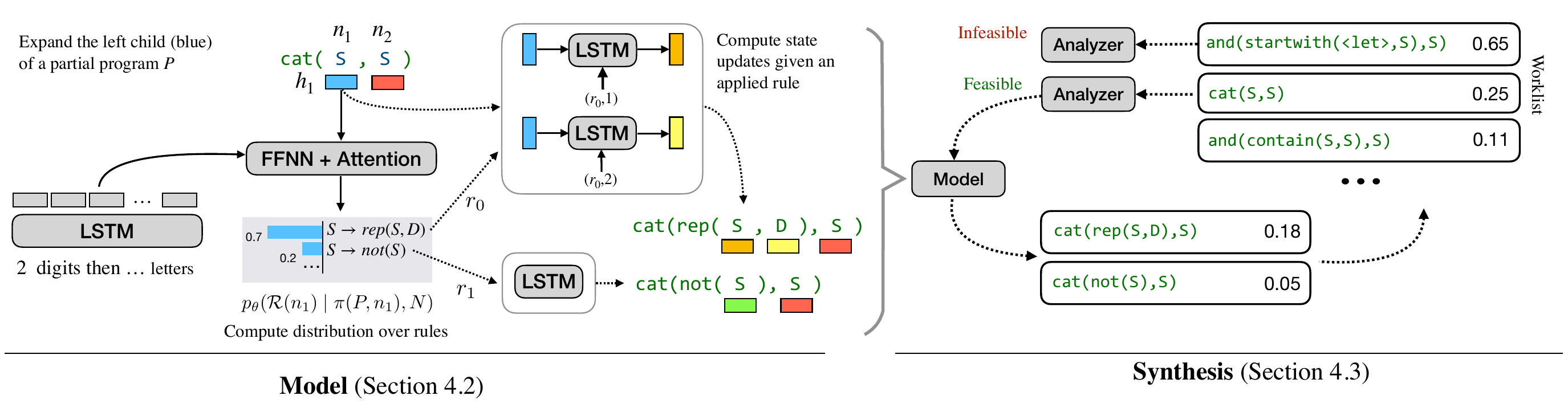}
\caption{Left: our neural model. A vector $h_i$ associated with a nonterminal is used to predict a distribution over grammar rules. Each rule instantiates new nonterminals which receive updated vectors based on LSTMs. Right: partial programs are taken from the worklist, analyzed to determine feasibility, and expanded, then the new partial programs are added to the worklist.}
\label{fig:asn}
\vspace{-0.2in}
\end{figure*}
\label{sec:algorithm}

We consider a class of models $M_\theta$ that  admit efficient optimal synthesis. Any model with the properties described in this section can be plugged into our synthesis algorithm (Section~\ref{sec:synthesis}). 

\begin{definition} \textbf{AST Path}
Given a node $n$ in a partial program $\partialprog$, we define the AST path  $\pi(\partialprog, n) = ((n_1,i_1), \ldots, (n_k,i_k))$ to be a sequence of pairs $(n_j, i_j)$ where (1) AST node $n_{j+1}$ is the $i_j$'th child of AST node $n_j$ and (2) the $i_k$'th child of $n_k$ is $n$. For instance, for  the partial program in Figure \ref{fig:ast_pp}, we have $\pi(P, n_4) = ((n_1, 1), (n_2, 1))$.
\end{definition}

\begin{definition} \textbf{Concrete/Inconcrete nodes}
Given a partial program $P$, we define the concrete nodes of $P$ as $\mathcal{C}(P)$ to be the nodes which have production rules assigned to them. The inconcrete nodes $\mathcal{I}(P)$ are the non-terminal leaf nodes whose production rules haven't been determined and need to be fill in in order to form a complete program.
\end{definition}

Given a partial program $P$, we define the probability of generating $P$ as the product of the probabilities of applying the productions labeling each node in the AST.  There are a number of possible ways we could factor and parameterize this distribution, including  PCFGs, where the distribution depends only on the parent, or as sequence models over a pre-order traversal of the tree~\citep{dong16, tranx,algolisp}. We choose the following factorization, similar to that used in Abstract Syntax Networks (ASN)~\citep{asn}, where a production rule depends on the derivation path leading to that nonterminal:
\begin{equation}
p_\theta(P \mid \nl)=\prod_{n\in \mathcal{C}(P) } p_\theta(\mathcal{R}(n) \mid \pi(P,n),\nl)
\end{equation}
The chief advantage of this factorization is that the score of a partial program is \textbf{invariant to the derivation order of that program}, assuming they were generated according to some topological ordering. Two derivations of the same tree $P$ that differ only in the order that child branches were generated are still assigned the same probability, allowing for flexibility in the search process. Second, for a partial program $P$, the distribution over rules of every unexpanded non-terminal leaf node does not depend on the others', which allows the estimation of the \emph{upper bound} (maximum possible probability) of the complete programs that can be derived from $P$. Specifically, we define the upper bound of the complete programs that can possibly be derived from a partial program $P$ as:

\begin{multline}
 u_\theta(P \mid \nl)=  p_\theta(P\mid N) \\ \prod_{n\in \mathcal{I}(P)} \max_r p_\theta(r \mid \pi(P,n),\nl). 
\end{multline}

This bound incorporates the known probabilities of concrete nodes as well as the minimum cost of filling inconcrete non-terminals, and thus more accurately estimates the cost of the optimal complete program given this partial program. A sequence model traversing the tree with a fixed order cannot estimate such an upper bound as the probabilities of inconcrete nodes are not known.


\subsection{Neural Model}
\label{sec:neural}
We instantiate the neural model defined above using a simplified version of ASN~\citep{asn}, which respects the $p_\theta(\mathcal{R}(n) \mid \pi(P,n),\nl)$ factorization for the production of each node in the tree. Figure~\ref{fig:asn} shows how ASN recursively computes the probability of labeling a node $n$ as $\mathcal{R}(n)$.

Consider the partial program \texttt{\small cat}($\mathcal{S}$($n_1$), $\mathcal{S}$($n_2$)); we need to define the probability distribution over legal productions on the first node $n_1$: $p_\theta(\mathcal{R}(n_1) \mid \pi(P,n),N) = p_\theta(\mathcal{R}(n_1) \mid \{(\texttt{\small{cat},1})\},N)$.
 
We encode the AST path using an LSTM \citep{lstm}. Define LSTM($h_0$,$(r_j,i_j)$) to be an LSTM with initial state $h_0$ and which,  at each timestep, consumes a tuple consisting of a node $n_j$ and a parent-child index $i_j$ (i.e., an element in $\pi(P,n)$).\footnote{This abstraction allows our LSTM to implement the hidden state computation of the ``constructor'' module from \citet{asn}. Our production rule model follows the ``primitive'' and ``composite type'' modules.} We embed each tuple $(n_j, i_j)$ by $W_{\mathcal{R}(n_j),i_j}$, where $W$ is specialized to the rule and position. Then: $h_{\text{root}} = \text{LSTM}(\nl)$ and $h_n=\text{LSTM}(h_{\text{root}},\pi(P,n))$
where $\text{LSTM}(\nl)$ denotes an encoding of the natural language input. The hidden state $h_n$ encodes both the user's NL specification as well as where we are in the parse tree, allowing us to model which grammar symbol should be likely at this position. 

Given this hidden state $h_n$, the probability for each production rule at node $n$ is computed using a feedforward neural network (FFNN) module and attention over the NL input:

\begin{multline*}
    p_\theta(\cdot \mid \pi(P,n),N)= \\ \text{softmax}(\text{FFNN}(h_n;\text{Attn}(h_n, \text{LSTM}(\nl))))
\end{multline*}

During search, each \texttt{Expand} operation instantiates a node $n$ with each possible rule according to the probabilities above, then computes the hidden states for any new nonterminals using the LSTM.

\subsection{Synthesis}
\label{sec:synthesis}

\begin{figure}
    \begin{algorithm}[H]
    \footnotesize
    \caption{Synthesis Algorithm}
    \begin{algorithmic}[1]
    \Procedure{$\textsc{OpSynth}$}{$\grammar, \spec, \nl, \model_\theta$}
    \Statex \Input{A CFG $\grammar = (V, \Sigma, R, S_0) $, specification $\spec$, natural language $\nl$ and model $\model_\theta$}
    \Statex \Output{Complete program $P$ with highest probability under $\model_\theta$ that satisfies $\spec$, or $\bot$ (no program exists)}
    \vspace{0.05in}
    \State $\worklist := \{(S_0, 1)\}$;
    \While{$\worklist \neq \emptyset$}
    \State $(\partialprog, \rho) := \worklist.\texttt{dequeue}()$; \Comment{upper bound $\rho$ associated with the partial program $P$}
    \If{\texttt{Infeasible}($\partialprog, \spec$)} \texttt{continue};
    \EndIf
    \If{\texttt{IsConcrete}($\partialprog$)} \Return $\partialprog$;
    \EndIf
    \State $l := \texttt{SelectLeaf}(\partialprog)$
    \For{$r \in \texttt{Supp}(\model_\theta(\pi(\partialprog, l), \nl))$}
    \State $\partialprog' := \texttt{Expand}(\partialprog, l, r)$
    \State $\worklist.\texttt{add}((\partialprog', u_\theta(P'|N))$
    \EndFor
    \EndWhile
    \State \Return $\bot$;
    \EndProcedure
    \end{algorithmic}
    \label{fig:algorithm}
    \end{algorithm}
\end{figure}

\begin{figure*}
\centering
\[
\begin{array}{cr}
\begin{array}{c}
\ \ \ \ \ \ 
\irule{
\begin{array}{c}
\mathsf{Root}(\partialprog) = n \quad  \mathcal{S}(n) \in V
\end{array}
}{
 \partialprog   \hookrightarrow (y = \top, y = \bot)
} \ \ {\rm {\text{(a)}}}
\\ \\ 
\irule{
\begin{array}{c}
\mathsf{Root}(\partialprog) = n   \quad 
n_i \in \mathsf{Children}(\partialprog) \quad 
\mathsf{Subtree}(\partialprog,n_i)   \hookrightarrow  (\psi_i^+(y, \mathbf{x}), \psi_i^-(y, \mathbf{x}))
\end{array}
}{
 \partialprog  \hookrightarrow (\exists \mathbf{z}. ( \Phi^+(\mathcal{S}(n))) \land \bigwedge_i \psi_i^+[z_i/y]), \exists \mathbf{z}. ( \Phi^-(\mathcal{S}(n))) \land \bigwedge_i \psi_i^-[z_i/y])
} \ \ {\rm {\text{(b)}}}   
\\ \\
\irule{
\begin{array}{c}
\partialprog \hookrightarrow(\psi^+(y, \mathbf{x}), \psi^-(y, \mathbf{x}))
\quad
\mathbf{UNSAT}(\bigwedge_{(\mathbf{i},o) \in \pose} \psi^+[o/y, \mathbf{i}/\mathbf{x}]  \land  \bigwedge_{(\mathbf{i},o) \in \nege} \neg \psi^-[o/y, \mathbf{i}/\mathbf{x}])

\end{array}
}{
P \not \models (\pose, \nege)
} \ \ {\rm {\text{(c)}}}
\end{array}
\end{array}
\]

\caption{Inference rules describing procedure {\sc Infeasible}($P, \phi$) for specification $\phi$ consisting of positive examples $\pose$ and negative examples $\nege$. Rules (a)-(b) of the form $\partialprog \hookrightarrow (\phi^+, \phi^-)$ generate a pair of logical formulas over- and under- approximating the semantics of partial program $\partialprog.$ The notation $\psi[z/y]$ denotes substituting variable $y$ with $z$ in formula $\psi$.}
\label{fig:encode} 
\end{figure*}

In this section, we describe a search algorithm to solve the optimal neural synthesis problem defined in Equation~\ref{eq:opt-synth}.

The key idea is to maintain a priority list $\worklist$ of partial programs, ranked according to the upper bound ($u_\theta(P)$) probability of the complete programs that can be derived from this partial program.
Then, in each iteration of the search procedure, we pick the highest upper bound partial program $\partialprog$ in $\worklist$, check its feasibility using  program analysis, and if it is feasible, expand one of the non-terminals in $\partialprog$ using the applicable CFG productions. Since complete programs are dequeued from $\worklist$ in decreasing order of their probability according to $\model_\theta$, the first complete program that satisfies $\phi$ is guaranteed to be optimal under $\model_\theta$ (proof in the in appendix); thus, our algorithm is guaranteed to return an optimal program if a solution exists.

\paragraph{Infeasibility pruning} Our top-down search allows us to exploit program analysis techniques to prune the search space, by determining whether $\partialprog$ is infeasible with respect to the user's hard constraints. A common way of doing this is to use well-known \emph{abstract interpretation} techniques from the programming languages literature to approximate  program semantics~\citep{cousot77, popa-book}. In particular, given a partial program $\partialprog$, the idea behind the feasibility checking procedure is to generate a pair of logical formulas $(\psi^+, \psi^-)$ over- and under-approximating $\partialprog$'s semantics respectively. If there is any positive example $e^+ \in \pose$ that is inconsistent with $\psi^+$, then the partial program is infeasible. Similarly, if there is any negative example $e^- \in \nege$ that satisfies $\psi^-$, we can again conclude that $\partialprog$ must be infeasible.


Figure~\ref{fig:encode} describes our feasibility checking procedure in terms of inference rules, where rules (a) and (b) generate a pair of over-  and under-approximations of the program, and rule (c) checks feasibility of these approximations with respect to the provided examples.
Here,  free variables $\mathbf{x}$ in the formula represent program inputs, and free variables $y$ represent the program output.  The existentially quantified variables $\mathbf{z}$ corresponds to values of sub-expressions.  The first rule states that ``holes" (i.e., non-terminals) in the partial program are over-approximated using $y=\top$ meaning the sub-program can return anything, and they are under-approximated using $y=\bot$, meaning that the sub-program returns nothing. 
The second rule is used to (recursively) construct an approximation of a sub-AST rooted at node $n$. This rule utilizes a pair of mappings $\Phi^+, \Phi^-$ where $\Phi^+$ (resp. $\Phi^-$) gives an over-approximating (resp. under-approximating) semantics for each language construct.  
In rule (b), each child formula $\psi^+_i, \psi_i^-$ must be satisfied as well as the parent formula, and these are unified by a shared set of new existentially-quantified variables.

The final rule uses the generated over- and under-approximations of the partial program to check feasibility. In particular, we conclude that the partial program is infeasible if there is any positive example $e^+ \in \pose$ that is inconsistent with $\psi^+$or  any negative example $e^- \in \nege$ that satisfies $\psi^-$.

\paragraph{Instantiation of the $\textsc{Infeasible}$ procedure for the regex domain}

Recall that $\textsc{Infeasible}$ prunes a given partial program $\partialprog$ by checking  consistency between the approximate program semantics and the given examples. In the regex domain, we encode the semantics of a regex in terms of the set of strings it can match.
To enable checking  consistency between a given example and the regex, given a string $s$, we use a program $\texttt{InLang}(s, \partialprog)$ (denoted as $\partialprog'$) to represent whether $s$ is in the set of strings that can be matched by $\partialprog$. 

As an example,
consider the partial program $\partialprog$: $\texttt{\small cat}(\texttt{or}(\texttt{\small <0>},V_1),\texttt{\small <1>})$. We encode the semantics of the program $\partialprog'$: $\texttt{InLang}(x, \partialprog)$ and ultimately end up with over- and under-approximations $(\psi^+, \psi^-)$ as follows:

\small
\begin{align*}
(\psi^+, \psi^-) =& (y \wedge (x \in  \texttt{cat}(\texttt{or}(\texttt{<0>},\top),\texttt{<1>})), \\& y \wedge (x \in  \texttt{cat}(\texttt{or}(\texttt{<0>},\bot),\texttt{<1>})))
\end{align*}
\normalsize
Intuitively, we've simply replaced the nonterminal $V_1$ by either $\top$ or $\bot$, indicating that all strings or no strings are matched by the eventual program at $V_1$. In this case, the approximation is simple, but in general it cannot just be written down intuitively. We produce it recursively using the procedure in Figure~\ref{fig:encode}, which yields the following intermediate over- and under-approximated formulas:


\small
\begin{align*}
    (\psi^+, \psi^-) =& ((\exists \textbf{z}. y \wedge (x \in z
    _0 \wedge \psi^+_0[z_0/y])), \\&\exists \textbf{z}.y \wedge ( x \in z_0 \wedge \psi^-_0[z_0/y]))) \\
    (\psi_0^+, \psi_0^-) =& (\exists \textbf{z}. y =  \texttt{cat}(z_1, z_2) \wedge \psi_1^+[z_1/y] \wedge \psi_2^+[z_2/y], \\&\exists \textbf{z}. y =  \texttt{cat}(z_1, z_2) \wedge \psi_1^-[z_1/y] \wedge \psi_2^-[z_2/y]) \\
    (\psi_1^+, \psi_1^-) =& (\exists \textbf{z}.y =  \texttt{or}(z_3, z_4) \wedge \psi_3^+[z_3, y] \wedge \psi_4^+[z_4/y], \\ & \exists \textbf{z}.y =  \texttt{or}(z_3, z_4) \wedge \psi_3^-[z_3, y] \wedge \psi_4^-[z_4/y]) \\
    (\psi_2^+, \psi_2^-) =& (y = \texttt{<1>}, y = \texttt{<1>}) \\
    (\psi_3^+, \psi_3^-) =& (y = \texttt{<0>}, y = \texttt{<0>}) \\
    (\psi_4^+, \psi_4^-) =& (y = \top, y = \bot)
\end{align*}
\normalsize

To confirm the utility of this representation, suppose we have a positive example $i = \texttt{"a1"}, o = \texttt{True}$ and a negative example $i = \texttt{"01"}, o = \texttt{True}$. According rule (c) of Figure \ref{fig:encode}, we check if the following formula is unsat:

\small
\begin{align*}
\texttt{True} \wedge (\texttt{"a1"} \in \texttt{cat}(\texttt{or}(\texttt{<0>},\top),\texttt{<1>})) \wedge \\ \neg (\texttt{True} \wedge (\texttt{"01"} \in \texttt{cat}(\texttt{or}(\texttt{<0>},\bot),\texttt{<1>})))    
\end{align*}
\normalsize

Since the under-approximated semantics of $\partialprog$ contains the string $\texttt{"01"}$, this formula is indeed unsat and we are able to prune this partial program.

\section{Experimental Setup}
\label{sec:experimental_setup}

We evaluate our synthesizer on  the English \streg{} dataset for multimodal synthesis of regular expressions. This dataset contains 3520 labeled examples, including an NL description, positive/negative examples, and the target regex.  We choose this dataset for our evaluation because (1) it is only the dataset containing both examples and NL where the NL description \emph{is written by humans}, and (2) this dataset is quite challenging, with existing techniques achieving under 50\% accuracy.

\paragraph{Implementation Details}
\label{sec:implementation}

As stated in Section~\ref{sec:neural}, our model is an Abstract Syntax Network tailored to fit the regex DSL used in \streg{}. We train our neural model to maximize the log likelihood of generating ground truth regexes given the NL using the Adam optimizer~\citep{adam}, stopping when the performance on dev set converges. More details are in the appendix. 

We implement the infeasibility checking procedure for our regex DSL by encoding the semantics of each operator in the theory of strings~\citep{strings}. Since all existentially quantified variables in the resulting formula can be eliminated through substitution, the resulting constraints are of the form $s \in R$ (or $s \not \in R$) where $s$ is a string constant and $R$ is a regular expression. Thus, we can check the satisfiability of these formulas using the Bricks library~\citep{automaton}. The appendix describes both the semantics of the DSL constructs as well as the rules used to generate the encoding a partial program,

Because of our infeasibility check, the order of expanding non-terminals can impact the efficiency of our search, as we want to prune any infeasible partial programs when they are less concrete.
We experimented with several methods of selecting a leaf node to expand, including pre-order traversal, choosing high-level nodes first, and choosing lowest-entropy nodes first. Pre-order traversal seemed to work best;
details about the expansion order can be found in the supplementary.

\begin{table*}[t]
    \centering
    \small
    \begin{tabular}{l|cccc|cccc}
    \toprule
        \multirow{2}{*}{Approach}  & \multicolumn{4}{c|}{Test} & \multicolumn{4}{c}{Test-E}  \\
         & \%Sol & \%Cons &  \#St &Time & \%Sol & \%Cons & \#St & Time \\
        \toprule
        AlphaRegex & \phantom{0}3.6 & 51.8 & 1.4\textsubscript{10\textsuperscript{6}} & 51.0 & \phantom{0}3.5 & 49.6 & 1.4\textsubscript{10\textsuperscript{6}} & 53.8 \\
        DeepCoder & \phantom{0}1.1 & \phantom{0}6.2 & 7.4\textsubscript{10\textsuperscript{4}} & 84.7 & \phantom{0}1.3 & \phantom{0}6.0 & 6.8\textsubscript{10\textsuperscript{4}} & 86.2\\
        RobustFill & \phantom{0}3.5 & 39.4 & 1.9\textsubscript{10\textsuperscript{3}} &21.1 & \phantom{0}3.5 & 38.4& 2.0\textsubscript{10\textsuperscript{3}} & 22.1\\
        \midrule
        \textsc{Sketch} & 45.2 &75.4 & 3.1\textsubscript{10\textsuperscript{3}} &18.4 & 29.8& 62.8 & 3.5\textsubscript{10\textsuperscript{3}} & 21.5\\
        \textsc{TreeSearch} & 48.7&  69.8 & $-$ &13.2 & 31.1 & 56.1 & $-$ & 19.1\\ 
        \midrule
        Seq2Seq\textsuperscript{+$\mathcal{P}$} & 48.2 &78.2 & 1.3\textsubscript{10\textsuperscript{4}} & 66.5 & 36.0 &64.3 & 1.5\textsubscript{10\textsuperscript{4}} & 76.8 \\
        TranX\textsuperscript{+$\mathcal{P}$} & 53.1 & 87.8 & 5.6\textsubscript{10\textsuperscript{3}} & 31.4 & 38.1 & 77.4& 6.4\textsubscript{10\textsuperscript{3}} & 36.1\\
        \midrule
        \textsc{ASN}\textsuperscript{+$\mathcal{P}$} & 58.0 &87.8 & 1.3\textsubscript{10\textsuperscript{3}} & 13.6 & 45.8 &78.2 & 1.4\textsubscript{10\textsuperscript{3}} & 15.1\\

        \textsc{OpSynth} & \bf 60.8 & \bf 88.4 & \bf 8.8\textsubscript{10\textsuperscript{2}} & \bf \phantom{0}9.5 & \bf 48.8 & \bf 80.9 & \bf 1.3\textsubscript{10\textsuperscript{3}} & \bf 14.2   \\
        \textsc{OpSynth}\textsuperscript{-$\mathcal{P}$} & 56.6 &78.5&  $-$ &13.8 & 44.7 & 67.0 & $-$ & 20.3\\
        
        \textsc{OpSynth}\textsuperscript{+$\mathcal{R}$} & 59.9& 88.2 & 8.8\textsubscript{10\textsuperscript{2}} & \phantom{0}9.9& 45.0 & 80.7 & 1.3\textsubscript{10\textsuperscript{3}} & 14.9\\
        \bottomrule
    \end{tabular}
         \caption{\streg{} results: fraction of solved benchmarks (\%Sol), fraction of benchmarks where we find an I/O-consistent program (\%Cons), average number of states explored (\#St), and average time in seconds.}
    \label{tab:main_comparison}
\end{table*}

\paragraph{Baselines}
We compare our method against three programming-by-example (\textbf{PBE-only}) baselines, \textsc{AlphaRegex} \citep{lee}, \textsc{DeepCoder} \citep{deepcoder}, and \textsc{RobustFill} \citep{robustfill}. \textsc{AlphaRegex} is an enumerative regex synthesizer that uses breadth-first search to find regexes that are consistent with the examples. Both \textsc{DeepCoder} and \textsc{RobustFill} are neural program synthesis approaches.
\textsc{DeepCoder} places a distribution over constructs and terminals based on examples, and uses this distribution to carry out DFS search, whereas \textsc{RobustFill} uses beam search to autoregressively build programs.

We further compare our method against prior multimodal program synthesis techniques, \textsc{Sketch} \citep{deepsketch} and \textsc{TreeSearch} \citep{algolisp} with appropriate tuning of the hyperparameters and the \textsc{Sketch} synthesizer for this setting. We do not compare against \textsc{SketchAdapt} \citep{infersketch} because it relies on the assumption that every program consistent with examples is the gold program, which does not hold in our setting. 


We also consider two NL-to-code models, Seq2Seq and TranX \citep{tranx}, which we modify to filter out partial programs that are inconsistent with the examples. Specifically, we adapt these baselines in a similar way as proposed in \citet{structregex} by filtering the beam at every timestep during search. Implementation details of all our baselines are in the appendix.

We refer to our Optimal Synthesis approach as  \textbf{\textsc{OpSynth}}. We also show ablations: \textsc{ASN}\textsuperscript{+$\mathcal{P}$} (ASN with our pruning during beam search), and \textsc{OpSynth}\textsuperscript{-$\mathcal{P}$} to further demonstrate the benefits of our approach over models like \cite{algolisp} that do not use such pruning. Finally, we also consider an extension denoted as  \textbf{\textsc{ OpSynth\textsuperscript{+$\mathcal{R}$}}}, which extends \textsc{OpSynth} with the \textsc{Attention A Model} from  \textsc{RobustFill} \citep{robustfill}, which encodes the examples $\phi$ using another set of LSTM layers. 
To combine these signals, we define the probability of applying rule $r$ on $n$ as:

\begin{multline*}
    p_\theta(r|n,P,N)=\text{softmax}(\text{FFNN}(h_n; \\ \ \ \ \ \ \ \text{Attn}(h_n, \text{context}(\nl); \text{Attn}(h_n, \text{context}(\spec)))).
\end{multline*}

\section{Results and Analysis}\label{sec:results}
In the following experiments, we evaluate our approach  based on two criteria: (1) accuracy, measured by the fraction of solved synthesis tasks, and (2) efficiency, measured by the number of partial programs searched and the run time.

\paragraph{Main Results}

Our main results are shown in 
Table~\ref{tab:main_comparison}.  We report results on two test sets from \streg{}; Test-E is annotated by a distinct set of annotators from the training set.

As shown in the top part of Table~\ref{tab:main_comparison}, pure PBE approaches do poorly on this dataset due to not utilizing NL. These approaches either fail to find a regex consistent with the examples within a time limit of 90 seconds or  the synthesized regex is semantically different  from the target one. These results from PBE-only approaches demonstrate the importance of using a model that places distributions over programs conditioned on the NL description.

The second and third parts of Table~\ref{tab:main_comparison} show results from prior multimodal neural synthesis approaches and NL-to-code models augmented with example-based pruning \citep{structregex}. \textsc{Sketch} slightly outperforms \textsc{TreeSearch}, solving 45\% and 30\% of the Test and Test-E set respectively. Seq2Seq\textsuperscript{+$\mathcal{P}$} and TranX\textsuperscript{+$\mathcal{P}$}, which perform beam search guided by these models but also check feasibility of partial programs before adding them to the beam, outperform these other techniques: TranX\textsuperscript{+$\mathcal{P}$} outperforms  Seq2Seq\textsuperscript{+$\mathcal{P}$} and solves 53\% of the benchmarks on Test and 38\% for Test-E.

The last part of  Table~\ref{tab:main_comparison} provides results about \toolname and its ablations. \toolname achieves a substantial improvement over TranX\textsuperscript{+$\mathcal{P}$} and is able to solve approximately 61\% of benchmarks in Test and 49\% in Test-E.  In addition to solving more benchmarks, \toolname also explores only a fraction of the states explored by TranX\textsuperscript{+$\mathcal{P}$}, leading to a speedup of more than $2.5\times$.

We now compare \toolname against three of its ablations. \textsc{OpSynth}\textsuperscript{-$\mathcal{P}$} does not use program analysis to prune infeasible partial programs (hence, we do not report explored states as a measure of runtime), and \textsc{ASN}\textsuperscript{+$\mathcal{P}$} is similar to \toolname except that it uses beam search (with beam size 20) combined with the same pruning technique. Both the program analysis component and optimal search are important: without these, we observe a deterioration in both accuracy and efficiency. The last row in  Table~\ref{tab:main_comparison} shows an extension of \toolname described in Section~\ref{sec:experimental_setup} where we incorporate the {\sc RobustFill} model. We find that {\sc RobustFill} is ineffective on its own, and incorporating it into our base synthesizer actually decreases performance. While such neural-guided PBE approaches ({\sc DeepCoder} \citep{deepcoder} and  {\sc RobustFill} \citep{robustfill}) have been successful in prior work, they do not appear to be effective on this challenging task, or not necessary in the presence of strong natural language hints. Additionally, these models both rely on millions of synthetic examples in the originally reported settings.

\paragraph{Optimality and efficiency.}

\begin{table}
    \centering
    \small
        \scalebox{0.90}{
        \begin{tabular}{lcccccc}
    \toprule
          & \%Opt & Gap & \%Sol  & \%Cons & \#St & Time \\
        \toprule
        Beam 5 & 50.4 & 1.11 & 39.0 & 65.1 &  \phantom{0}290 & \phantom{0}3.3\\
        Beam 10 & 59.4 & 1.08 & 42.8 & 72.2 & \phantom{0}660 & \phantom{0}6.8 \\
        Beam 15 & 63.2 & 0.84 & 44.1 & 76.8 & 1040 & 11.0\\
        Beam 20 & 66.2 &  0.69 & 45.8 & 78.2 & 1430 & 15.1 \\ 
        \midrule
        OpSynth & \bf 80.9 & \bf \phantom{0}0.0  & \bf 48.9 & \bf 80.9 & \bf 1320 & \bf 14.2 \\
        \bottomrule
    \end{tabular}
        }
        \caption{Comparison between \toolname and a beam search-based alternative with the same model. }
    \label{tab:optimality}
\end{table}

\begin{figure}[t]
    \centering
    \includegraphics[width=\linewidth,trim=0 20 0 20,clip]{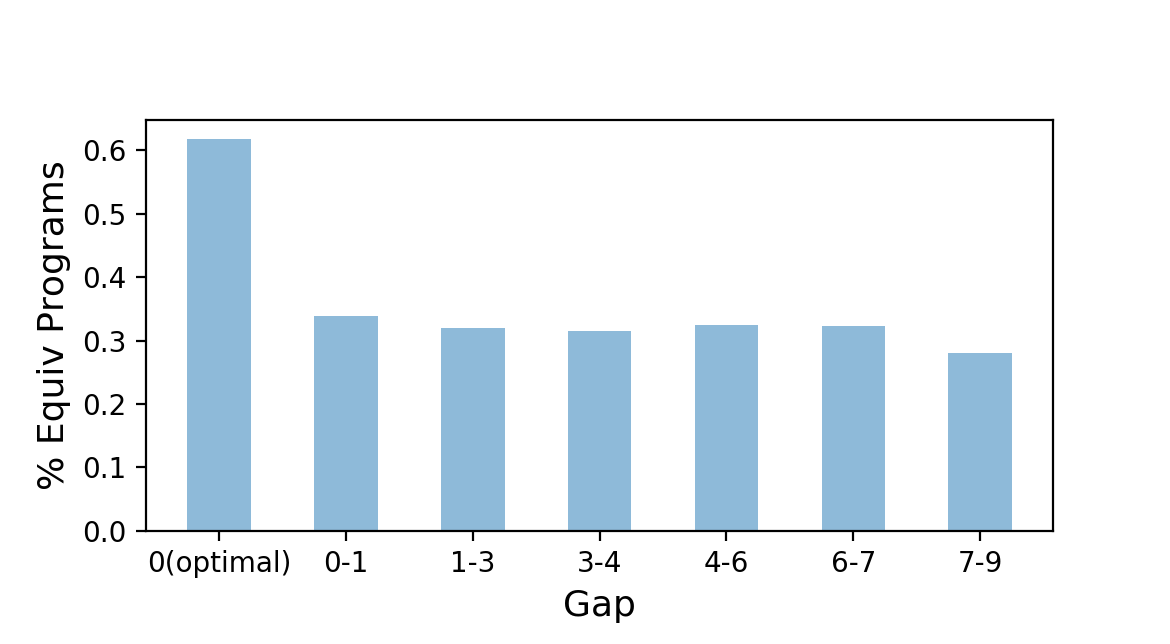}
    \caption{Fraction of programs equivalent to target regex based on score gap with the model-optimal program.}
    \label{fig:acc_by_gap}
\end{figure}

We now explore the benefits of optimality in more detail. Specifically, Table~\ref{tab:optimality} compares \toolname with an alternative that performs beam search with varying beam sizes for  Test-E. For  the purposes of this experiment, we terminate \toolname's search after it has explored  a maximum of 5000 states. For beam search, we terminate search when the beam is filled up with complete programs or the size of partial programs in the beam exceeds a threshold.

In Table~\ref{tab:optimality}, the column labeled ``\% Opt'' shows the percentage of optimal programs found by the search algorithm. We also show the gap (difference of log probability) between the best-scored programs found by each approach and the optimal programs; this is reported in the column labeled ``Gap''.  Finally, the last three columns show the fraction of solved instances (accuracy), the fraction of programs consistent with the examples, and the  number of explored states respectively.


As seen in Table~\ref{tab:optimality}, our optimal synthesizer finds the optimal program in 80.9\% of cases and solves 46.9\% of instances after exploring 810 states on average. Beam search with a beam size of 20 only finds 66.2\% optimal programs and solves fewer instances (45.8\%)  despite exploring more states. 

We further  evaluate the benefit of finding model-optimal programs in Figure~\ref{fig:acc_by_gap}. Here, we focus only on those programs that are consistent with the input-output examples. The x-axis shows the score gap from the optimal program, and the y-axis shows the percent of programs that are functionally equivalent to the desired regex.  As shown in Figure~\ref{fig:acc_by_gap}, 62\% of optimal programs are equivalent to the target regex, whereas only around 30\% of the nearly-optimal programs functionally match the ground truth.

\begin{figure}[t]
    \centering
    \includegraphics[width=\linewidth, trim=0 0 0 20,clip]{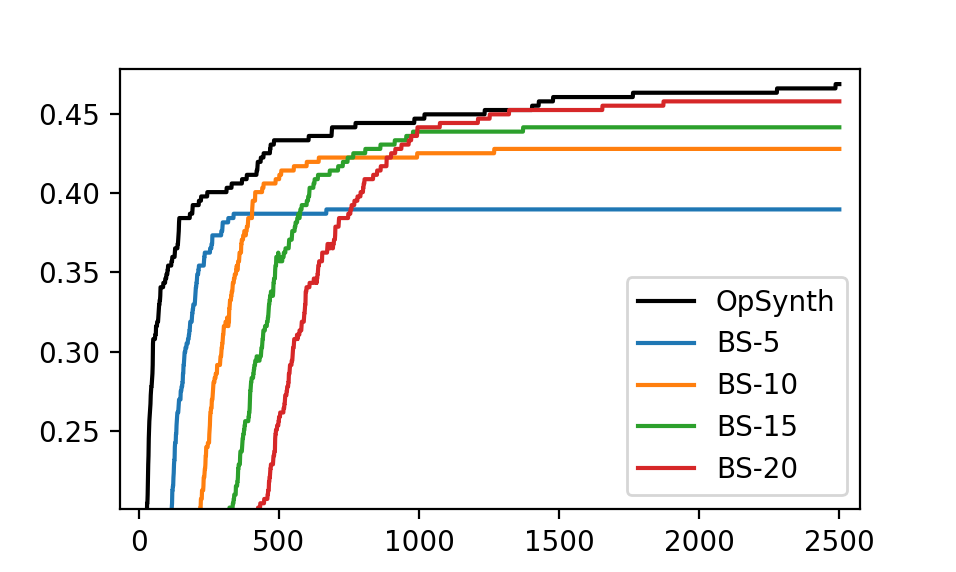}
         \caption{Fraction of solved instances versus the number of explored states. }
    \label{fig:efficiency}
\end{figure}


Finally,  Figure~\ref{fig:efficiency} plots the fraction of solved instances with respect to the number of states explored. \toolname consistently solves more instances than the other methods given the same budget without requiring a pre-specified beam size. 

\section{Related Work}
\paragraph{Natural Language to Logical Forms}
Semantic parsing (translating NL to executable logical forms) has been a long-standing research problem in the NLP community \citep{geoquery, atis}. Traditional grammar-based semantic parsers can construct database queries \citep{geoquery, atis}, lambda calculus expressions \citep{PCG} and programs in other DSLs \citep{kb13, overnight}. Recent advances in deep learning have explored seq2seq \citep{jialiang15} or seq2tree models \citep{dong16} that directly translate the NL into a logical form, and syntax-based models \citep{tranx} can also inject syntactic constraints. 
Our approach relies on similar neural modeling to predict the distribution of target programs from NL. However, search is much more complex in our example-guided synthesis setting, whereas prior neural semantic parsers approximate the best solution using beam search \citep{dong16, tranx}.

\paragraph{Optimal Synthesis with Examples} Prior work on PBE considers various notions of optimality using cost functions \citep{metasketch,lambda2,stoke} and machine learning \citep{menon13}. The first line of work allows users to specify the desired properties of the synthesized program; for instance, smaller program size, lower execution time, or more efficient memory usage. \citet{menon13} define optimality as the most likely constructs given a set of examples under a probabilistic context free grammar. 
In this work, we focus on a new setting where we guarantee the optimality with respect to a neural modal, which can encode specifications such as natural language that are hard to formulate as simple cost functions.



\paragraph{Multimodal Program Synthesis} There has been recent interest in synthesizing programs using a combination of natural language and examples \citep{algolisp,mars,infersketch,andreas2018,gulwanimulti}.
Specifically, \citet{regel} and \citet{deepsketch} parse the natural language into an intermediate representation and then use it to guide enumeration, but they do not have any optimality guarantees with respect to the neural model. \citet{spoc} synthesize programs by performing line-by-line translation of pseudocode to code and verify consistency with test cases at the end.  However, unlike our approach, their technique enumerates syntactically ill-formed programs, which they address using compiler error localization.

\section{Conclusion}

In this paper, we presented a technique for optimal synthesis from multimodal specifications. On a benchmark of complex regex synthesis problems, we showed that this approach is substantially more accurate than past models, and our synthesis algorithm finds the model-optimal program more frequently compared to beam search.

While we have evaluated this method on regular expressions, our technique is general and can be applied to other classic PBE domains on which powerful abstract interpretation techniques for feasibility checking are available, such as table transformations \cite{FengEtAl2017}, tensor and string manipulations \cite{wang2017program}, and other data wrangling \cite{neo}. Our technique, especially the notion of optimality with respect to a model, can also be valuable for more general program synthesis \cite{alet2021large,austin2021program} if pruning techniques can be developed for the particular tasks to be performed. In particular, \citet{austin2021program} note that large language models are not good at modeling execution semantics of programs; we see our execution-guided pruning techniques as a path forward in this domain.

\section*{Acknowledgments}

We thank the anonymous reviewers for their valuable feedback. This work was partially supported by NSF Grant IIS-1814522, NSF Grant SHF-1762299, NSF Award CCF-1811865, and a gift from Salesforce Inc. 

\bibliography{anthology,custom}

\begin{thebibliography}{47}
\expandafter\ifx\csname natexlab\endcsname\relax\def\natexlab#1{#1}\fi

\bibitem[{Alet et~al.(2021)Alet, Lopez-Contreras, Koppel, Nye, Solar-Lezama,
  Lozano-Perez, Kaelbling, and Tenenbaum}]{alet2021large}
Ferran Alet, Javier Lopez-Contreras, James Koppel, Maxwell Nye, Armando
  Solar-Lezama, Tomas Lozano-Perez, Leslie Kaelbling, and Joshua Tenenbaum.
  2021.
\newblock A large-scale benchmark for few-shot program induction and synthesis.
\newblock In \emph{Proceedings of the International Conference on Machine
  Learning (ICML)}.

\bibitem[{Andreas et~al.(2018)Andreas, Klein, and Levine}]{andreas2018}
Jacob Andreas, Dan Klein, and Sergey Levine. 2018.
\newblock {Learning with Latent Language}.
\newblock In \emph{Proceedings of the Conference of the North {A}merican
  Chapter of the Association for Computational Linguistics: Human Language
  Technologies (NAACL)}.

\bibitem[{Austin et~al.(2021)Austin, Odena, Nye, Bosma, Michalewski, Dohan,
  Jiang, Cai, Terry, Le, and Sutton}]{austin2021program}
Jacob Austin, Augustus Odena, Maxwell Nye, Maarten Bosma, Henryk Michalewski,
  David Dohan, Ellen Jiang, Carrie Cai, Michael Terry, Quoc Le, and Charles
  Sutton. 2021.
\newblock Program synthesis with large language models.
\newblock \emph{arXiv preprint arXiv:2108.07732}.

\bibitem[{Balog et~al.(2017)Balog, Gaunt, Brockschmidt, Nowozin, and
  Tarlow}]{deepcoder}
M~Balog, AL~Gaunt, M~Brockschmidt, S~Nowozin, and D~Tarlow. 2017.
\newblock Deepcoder: Learning to write programs.
\newblock In \emph{Proceedings of the International Conference on Learning
  Representations (ICLR)}.

\bibitem[{Bornholt et~al.(2016)Bornholt, Torlak, Grossman, and
  Ceze}]{metasketch}
James Bornholt, Emina Torlak, Dan Grossman, and Luis Ceze. 2016.
\newblock Optimizing synthesis with metasketches.
\newblock In \emph{Proceedings of the 43rd Annual ACM SIGPLAN-SIGACT Symposium
  on Principles of Programming Languages (POPL)}.

\bibitem[{Chen et~al.(2020)Chen, Wang, Ye, Durrett, and Dillig}]{regel}
Qiaochu Chen, Xinyu Wang, Xi~Ye, Greg Durrett, and Isil Dillig. 2020.
\newblock Multi-modal synthesis of regular expressions.
\newblock In \emph{Proceedings of the ACM SIGPLAN Conference on Programming
  Language Design and Implementation (PLDI)}.

\bibitem[{Chen et~al.(2019{\natexlab{a}})Chen, Liu, and Song}]{song1}
Xinyun Chen, Chang Liu, and Dawn Song. 2019{\natexlab{a}}.
\newblock Execution-guided neural program synthesis.
\newblock In \emph{Proceedings of the International Conference on Learning
  Representations (ICLR)}.

\bibitem[{Chen et~al.(2019{\natexlab{b}})Chen, Martins, and Feng}]{mars}
Yanju Chen, Ruben Martins, and Yu~Feng. 2019{\natexlab{b}}.
\newblock Maximal multi-layer specification synthesis.
\newblock In \emph{Proceedings of the 2019 27th ACM Joint Meeting on European
  Software Engineering Conference and Symposium on the Foundations of Software
  Engineering (FSE)}.

\bibitem[{Cousot and Cousot(1977)}]{cousot77}
Patrick Cousot and Radhia Cousot. 1977.
\newblock Abstract interpretation: a unified lattice model for static analysis
  of programs by construction or approximation of fixpoints.
\newblock In \emph{Proceedings of the 4th ACM SIGACT-SIGPLAN Symposium on
  Principles of Programming Languages (POPL)}.

\bibitem[{Desai et~al.(2016)Desai, Gulwani, Hingorani, Jain, Karkare, Marron,
  R, and Roy}]{gulwaninl}
Aditya Desai, Sumit Gulwani, Vineet Hingorani, Nidhi Jain, Amey Karkare, Mark
  Marron, Sailesh R, and Subhajit Roy. 2016.
\newblock Program synthesis using natural language.
\newblock In \emph{Proceedings of the 38th International Conference on Software
  Engineering (ICSE)}.

\bibitem[{Devlin et~al.(2017)Devlin, Uesato, Bhupatiraju, Singh, Mohamed, and
  Kohli}]{robustfill}
Jacob Devlin, Jonathan Uesato, Surya Bhupatiraju, Rishabh Singh, Abdel-rahman
  Mohamed, and Pushmeet Kohli. 2017.
\newblock {Robustfill: Neural Program Learning under Noisy I/O}.
\newblock In \emph{Proceedings of the International Conference on Machine
  Learning (ICML)}.

\bibitem[{Dong and Lapata(2016)}]{dong16}
Li~Dong and Mirella Lapata. 2016.
\newblock Language to logical form with neural attention.
\newblock In \emph{Proceedings of the Annual Meeting of the Association for
  Computational Linguistics (ACL)}.

\bibitem[{Ellis et~al.(2019)Ellis, Nye, Pu, Sosa, Tenenbaum, and
  Solar-Lezama}]{repl}
Kevin Ellis, Maxwell Nye, Yewen Pu, Felix Sosa, Josh Tenenbaum, and Armando
  Solar-Lezama. 2019.
\newblock Write, execute, assess: Program synthesis with a repl.
\newblock In \emph{Proceedings of the Conference on Advances in Neural
  Information Processing Systems (NeurIPS)}.

\bibitem[{Feng et~al.(2018)Feng, Martins, Bastani, and Dillig}]{neo}
Yu~Feng, Ruben Martins, Osbert Bastani, and Isil Dillig. 2018.
\newblock {Program Synthesis Using Conflict-driven Learning}.
\newblock In \emph{Proceedings of the ACM SIGPLAN Conference on Programming
  Language Design and Implementation (PLDI)}.

\bibitem[{Feng et~al.(2017)Feng, Martins, Van~Geffen, Dillig, and
  Chaudhuri}]{FengEtAl2017}
Yu~Feng, Ruben Martins, Jacob Van~Geffen, Isil Dillig, and Swarat Chaudhuri.
  2017.
\newblock {Component-Based Synthesis of Table Consolidation and Transformation
  Tasks from Examples}.
\newblock In \emph{Proceedings of the ACM SIGPLAN Conference on Programming
  Language Design and Implementation (PLDI)}.

\bibitem[{Feser et~al.(2015)Feser, Chaudhuri, and Dillig}]{lambda2}
John~K. Feser, Swarat Chaudhuri, and Isil Dillig. 2015.
\newblock Synthesizing data structure transformations from input-output
  examples.
\newblock In \emph{Proceedings of the 36th ACM SIGPLAN Conference on
  Programming Language Design and Implementation (PLDI)}.

\bibitem[{Hochreiter and Schmidhuber(1997)}]{lstm}
Sepp Hochreiter and J\"{u}rgen Schmidhuber. 1997.
\newblock \href {https://doi.org/10.1162/neco.1997.9.8.1735} {{Long Short-Term
  Memory}}.
\newblock \emph{Neural Comput.}, 9(8):1735–1780.

\bibitem[{Jia and Liang(2016)}]{jialiang15}
Robin Jia and Percy Liang. 2016.
\newblock Data recombination for neural semantic parsing.
\newblock In \emph{Proceedings of the Annual Meeting of the Association for
  Computational Linguistics (ACL)}.

\bibitem[{Kalyan et~al.(2018)Kalyan, Mohta, Polozov, Batra, Jain, and
  Gulwani}]{polozov1}
Ashwin Kalyan, Abhishek Mohta, Oleksandr Polozov, Dhruv Batra, Prateek Jain,
  and Sumit Gulwani. 2018.
\newblock Neural-guided deductive search for real-time program synthesis from
  examples.
\newblock In \emph{Proceedings of the International Conference on Learning
  Representations (ICLR)}.

\bibitem[{Kingma and Ba(2015)}]{adam}
Diederik~P. Kingma and Jimmy Ba. 2015.
\newblock Adam: {A} method for stochastic optimization.
\newblock In \emph{Proceedings of the International Conference on Learning
  Representations (ICLR)}.

\bibitem[{Kulal et~al.(2019)Kulal, Pasupat, Chandra, Lee, Padon, Aiken, and
  Liang}]{spoc}
Sumith Kulal, Panupong Pasupat, Kartik Chandra, Mina Lee, Oded Padon, Alex
  Aiken, and Percy~S Liang. 2019.
\newblock Spoc: Search-based pseudocode to code.
\newblock In \emph{Proceedings of the Conference on Advances in Neural
  Information Processing Systems (NeurIPS)}.

\bibitem[{Kushman and Barzilay(2013)}]{kb13}
Nate Kushman and Regina Barzilay. 2013.
\newblock {Using Semantic Unification to Generate Regular Expressions from
  Natural Language}.
\newblock In \emph{Proceedings of the Conference of the North {A}merican
  Chapter of the Association for Computational Linguistics: Human Language
  Technologies (NACCL)}.

\bibitem[{Lee et~al.(2016)Lee, So, and Oh}]{lee}
Mina Lee, Sunbeom So, and Hakjoo Oh. 2016.
\newblock {Synthesizing Regular Expressions from Examples for Introductory
  Automata Assignments}.
\newblock In \emph{Proceedings of the ACM SIGPLAN International Conference on
  Generative Programming: Concepts and Experiences (GPCE)}.

\bibitem[{Liang et~al.(2014)Liang, Reynolds, Tinelli, Barrett, and
  Deters}]{strings}
Tianyi Liang, Andrew Reynolds, Cesare Tinelli, Clark Barrett, and Morgan
  Deters. 2014.
\newblock A dpll (t) theory solver for a theory of strings and regular
  expressions.
\newblock In \emph{International Conference on Computer Aided Verification
  (CAV)}, pages 646--662. Springer.

\bibitem[{Luong et~al.(2015)Luong, Pham, and Manning}]{attention}
Thang Luong, Hieu Pham, and Christopher~D. Manning. 2015.
\newblock {Effective Approaches to Attention-based Neural Machine Translation}.
\newblock In \emph{Proceedings of the Conference on Empirical Methods in
  Natural Language Processing (EMNLP)}.

\bibitem[{Menon et~al.(2013)Menon, Tamuz, Gulwani, Lampson, and
  Kalai}]{menon13}
Aditya~Krishna Menon, Omer Tamuz, Sumit Gulwani, Butler Lampson, and
  Adam~Tauman Kalai. 2013.
\newblock A machine learning framework for programming by example.
\newblock In \emph{Proceedings of the International Conference on Machine
  Learning (ICML)}.

\bibitem[{M\o{}ller(2017)}]{automaton}
Anders M\o{}ller. 2017.
\newblock dk.brics.automaton -- finite-state automata and regular expressions
  for {Java}.
\newblock \texttt{http://www.brics.dk/automaton/}.

\bibitem[{Nielson et~al.(2015)Nielson, Nielson, and Hankin}]{popa-book}
Flemming Nielson, Hanne~R Nielson, and Chris Hankin. 2015.
\newblock \emph{Principles of program analysis}.
\newblock Springer.

\bibitem[{Nye et~al.(2019)Nye, Hewitt, Tenenbaum, and
  Solar-Lezama}]{infersketch}
Maxwell Nye, Luke Hewitt, Joshua Tenenbaum, and Armando Solar-Lezama. 2019.
\newblock Learning to infer program sketches.
\newblock In \emph{Proceedings of the International Conference on Machine
  Learning (ICML)}, pages 4861--4870.

\bibitem[{Polosukhin and Skidanov(2018)}]{algolisp}
Illia Polosukhin and Alexander Skidanov. 2018.
\newblock Neural program search: Solving programming tasks from description and
  examples.
\newblock In \emph{Workshop at the International Conference on Learning
  Representations (ICLR Workshop)}.

\bibitem[{Price(1990)}]{atis}
Patti Price. 1990.
\newblock Evaluation of spoken language systems: The atis domain.
\newblock In \emph{Proceedings of the DARPA Workshop on Speech and Natural
  Language}.

\bibitem[{Rabinovich et~al.(2017)Rabinovich, Stern, and Klein}]{asn}
Maxim Rabinovich, Mitchell Stern, and Dan Klein. 2017.
\newblock Abstract syntax networks for code generation and semantic parsing.
\newblock In \emph{Proceedings of the Annual Meeting of the Association for
  Computational Linguistics (ACL)}.

\bibitem[{Raza et~al.(2015)Raza, Gulwani, and Milic-Frayling}]{gulwanimulti}
Mohammad Raza, Sumit Gulwani, and Natasa Milic-Frayling. 2015.
\newblock Compositional program synthesis from natural language and examples.
\newblock In \emph{Proceedings of the International Joint Conference on
  Artificial Intelligence (IJCAI)}.

\bibitem[{Schkufza et~al.(2013)Schkufza, Sharma, and Aiken}]{stoke}
Eric Schkufza, Rahul Sharma, and Alex Aiken. 2013.
\newblock \href {https://doi.org/10.1145/2499368.2451150} {Stochastic
  superoptimization}.
\newblock \emph{SIGPLAN Not.}, 48(4):305–316.

\bibitem[{Shin et~al.(2018)Shin, Polosukhin, and Song}]{song2}
Eui~Chul Shin, Illia Polosukhin, and Dawn Song. 2018.
\newblock Improving neural program synthesis with inferred execution traces.
\newblock In \emph{Proceedings of the Conference on Advances in Neural
  Information Processing Systems (NeurIPS)}, pages 8917--8926.

\bibitem[{Sun et~al.(2020)Sun, Zhu, Xiong, Sun, Mou, and
  Zhang}]{sun2020treegen}
Zeyu Sun, Qihao Zhu, Yingfei Xiong, Yican Sun, Lili Mou, and Lu~Zhang. 2020.
\newblock Treegen: A tree-based transformer architecture for code generation.
\newblock In \emph{Proceedings of the Association for the Advancement of
  Artificial Intelligence (AAAI)}, pages 8984--8991.

\bibitem[{Wang et~al.(2018)Wang, Huang, Polozov, Brockschmidt, and
  Singh}]{robsql}
Chenglong Wang, Po-Sen Huang, Alex Polozov, Marc Brockschmidt, and Rishabh
  Singh. 2018.
\newblock Execution-guided neural program decoding.
\newblock In \emph{the Workshop on Neural Abstract Machines and Program
  Induction (NAMPI)}.

\bibitem[{Wang et~al.(2017)Wang, Dillig, and Singh}]{wang2017program}
Xinyu Wang, Isil Dillig, and Rishabh Singh. 2017.
\newblock Program synthesis using abstraction refinement.
\newblock In \emph{Proceedings of the ACM SIGPLAN-SIGACT Symposium on
  Principles of Programming Languages (POPL)}.

\bibitem[{Wang et~al.(2015)Wang, Berant, and Liang}]{overnight}
Yushi Wang, Jonathan Berant, and Percy Liang. 2015.
\newblock {Building a Semantic Parser Overnight}.
\newblock In \emph{Proceedings of the Annual Meeting of the Association for
  Computational Linguistics (ACL)}.

\bibitem[{Yaghmazadeh et~al.(2017)Yaghmazadeh, Wang, Dillig, and
  Dillig}]{sqlizer}
Navid Yaghmazadeh, Yuepeng Wang, Isil Dillig, and Thomas Dillig. 2017.
\newblock {SQLizer: Query Synthesis from Natural Language}.
\newblock In \emph{Proceedings of the ACM SIGPLAN International Conference on
  Object-Oriented Programming, Systems, Languages, and Applications (OOPSLA)}.

\bibitem[{Ye et~al.(2020{\natexlab{a}})Ye, Chen, Dillig, and
  Durrett}]{structregex}
Xi~Ye, Qiaochu Chen, Isil Dillig, and Greg Durrett. 2020{\natexlab{a}}.
\newblock Benchmarking multimodal regex synthesis with complex structures.
\newblock In \emph{Proceedings of the Annual Meeting of the Association for
  Computational Linguistics (ACL)}.

\bibitem[{Ye et~al.(2020{\natexlab{b}})Ye, Chen, Wang, Dillig, and
  Durrett}]{deepsketch}
Xi~Ye, Qiaochu Chen, Xinyu Wang, Isil Dillig, and Greg Durrett.
  2020{\natexlab{b}}.
\newblock {Sketch-Driven Regular Expression Generation from Natural Language
  and Examples}.
\newblock In \emph{Transactions of the Association for Computational
  Linguistics (TACL)}.

\bibitem[{Yin et~al.(2018)Yin, Deng, Chen, Vasilescu, and Neubig}]{conala}
Pengcheng Yin, Bowen Deng, Edgar Chen, Bogdan Vasilescu, and Graham Neubig.
  2018.
\newblock Learning to mine aligned code and natural language pairs from stack
  overflow.
\newblock In \emph{2018 IEEE/ACM 15th International Conference on Mining
  Software Repositories (MSR)}.

\bibitem[{Yin and Neubig(2017)}]{tranx}
Pengcheng Yin and Graham Neubig. 2017.
\newblock A syntactic neural model for general-purpose code generation.
\newblock In \emph{Proceedings of the Annual Meeting of the Association for
  Computational Linguistics (ACL)}.

\bibitem[{Yin and Neubig(2018)}]{tranxdemo}
Pengcheng Yin and Graham Neubig. 2018.
\newblock {TRANX}: A transition-based neural abstract syntax parser for
  semantic parsing and code generation.
\newblock In \emph{Proceedings of the Conference on Empirical Methods in
  Natural Language Processing: System Demonstrations (EMNLP)}.

\bibitem[{Zelle and Mooney(1996)}]{geoquery}
John~M. Zelle and Raymond~J. Mooney. 1996.
\newblock Learning to parse database queries using inductive logic programming.
\newblock In \emph{Proceedings of the Association for the Advancement of
  Artificial Intelligence (AAAI)}.

\bibitem[{Zettlemoyer and Collins(2005)}]{PCG}
Luke~S. Zettlemoyer and Michael Collins. 2005.
\newblock Learning to map sentences to logical form: Structured classification
  with probabilistic categorial grammars.
\newblock In \emph{Proceedings of the Conference on Uncertainty in Artificial
  Intelligence (UAI)}.

\end{thebibliography}
\bibliographystyle{acl_natbib}

\newpage
\appendix
\section{Guarantee of Optimality}
\begin{theorem} [Guarantee of Optimality]
Suppose given a CFG $\grammar = (V, \Sigma, R, S_0) $, specification $\spec$, natural language $\nl$ and model $\model_\theta$, $\textsc{OpSynth}$ returns a program $\partialprog^*$. Then, for any program $\partialprog \models \phi$, $\model_\theta(\partialprog)\leq\model_\theta(\partialprog^*)$.
\end{theorem}

\begin{proof}
Assume $\partialprog^*$ is the returned program of $\textsc{OpSynth}(\grammar, \spec, \nl, \model_\theta)$ and there exits a program $\partialprog$ such that $\partialprog \models \phi$ and $\model_\theta(\partialprog)>\model_\theta(\partialprog^*)$. Since $\model_\theta(\partialprog)>\model_\theta(\partialprog^*)$, $P$ must have been present in the worklist and considered as a concrete program before the model visited $P^*$. But then, given that $\partialprog \models \phi$, then $\textsc{OpSynth}$ will return $\partialprog$ rather than $\partialprog ^ *$, which contradicts the assumption.
\end{proof}
\section{CFG for Regular Expressions}
We present the CFG for the regex domain language taken from \textsc{StructuredRegex}~\citep{structregex} in Figure~\ref{fig:regex_cfg}. Its correspondence to the constructions in the standard regular expression is shown in the Appendix A of \citet{structregex}.

\begin{figure}[t]
    \centering
    \small
    \begin{align*}
        S_0 \rightarrow& \ V_1 \\
        V_1 \rightarrow& \ T_1 \ | \ {\tt startwith}(V_1) \ | \ {\tt endwith}(V_1) \ | \ {\tt contain} (V_1) \\ 
        |& \ {\tt not}(V_1) \ | \ {\tt and}(V_1, V_1) \ | \ {\tt or}(V_1, V_1) \\
        |& \ {\tt optional}(V_1) \ | \ {\tt star}(V_1) \\
        |& \ \text{{\tt concat}\footnotemark} (V_1, V_1) \ | \ {\tt repeat}(V_1, k) \\
        |& \ {\tt repatleast}(V_1, k) \ | \ {\tt reprange}(V_1, k_1, k_2) \\ 
        T_1 \rightarrow& \  c  \ | \ \texttt{<let>} \ | \ \texttt{<cap>} \ | \ \texttt{<low>} \\
           | & \ \texttt{<num>} \ | \ \texttt{<any>} \ | \ \texttt{<spec>} \ | \ \texttt{<null>} \\
    \end{align*}
    \caption{Regex CFG. Here $k \in \mathbb{Z}^+$ and $c$ is a character class, such as $\texttt{<a>}$, $\texttt{<1>}$, etc. }
    \label{fig:regex_cfg}
\end{figure}
\footnotetext{We note {\tt concat} as {\tt cat} in the paper.}

\section{Encoding for the $\textsc{Infeasible}$ Procedure for Regex}
\begin{figure*}[t]
    \centering
    \small
    \begin{align*}
    \ & \ \Phi^{\{+,-\}}(\texttt{InLang},y, \mathbf{x}, \mathbf{z}) = (y \wedge (\mathbf{x} \in z_0 ))  \\
    f \in \{\texttt{startwith}, \texttt{endwith}, \texttt{contain}, \texttt{not}, \texttt{optional}, \texttt{star}\} \ & \ \Phi^{\{+,-\}}(f, y, \mathbf{z}) = (y =  f(z_1)) \\
    f \in \{\texttt{cat}, \texttt{and}, \texttt{or}, \texttt{repeat}, \texttt{repatleast}\}  \ & \ \Phi^{\{+,-\}}(f, y, \mathbf{z}) = (y =  f(z_1, z_2)) \\
    f \in \{\texttt{reprange}\}  \ & \ \Phi^{\{+,-\}}(f, y, \mathbf{z}) = (y = f(z_1, z_2, z_3)) \\
    \end{align*}
    \caption{$\Phi^{+,-}$ in the regex domain. Here we omit the $T_1$ and $k$ case. The encoding for non-terminal symbols is rule (a) in Figure~\ref{fig:encode} where $\top = \texttt{star}(\texttt{<any>})$ and $\bot  = \texttt{<null>}$. }\label{fig:encoderegex}
\end{figure*}


We describe our detailed instantiation of the $\textsc{Infeasible}$ procedure described in Section \ref{sec:synthesis} in the regex domain. Recall that we encode the semantics of a regex in terms of the set of strings it can match, and we use the program $\texttt{InLang}(s, \partialprog)$ (denoted as $\partialprog'$) to represent whether $s$ is in the set of strings that can be matched by $\partialprog$.  To encode a program $\partialprog'$ for consistency checking, we use the set of encoding rules presented in Figure~\ref{fig:encoderegex} to generate its over- and under- approximated semantics. In the regex domain, for most of the constructs, we can model the precise semantics except for the non-terminal symbols in the partial program.

\section{Neural Model Details}
 As described in Section~\ref{sec:neural}, our neural model resembles an Abstract Syntax Network~\citep{asn} tailored to fit the regex DSL used in \streg{}. We show the grammar in Figure~\ref{fig:grammar}. As there is no production rule having optional or sequential cardinality, we do not include the ``constructor field module'' from the ASN in our implementation. We encode the NL using a single-layer Bi-LSTM encoder with a hidden state size of 100. In the decoding phase, we set the size of the hidden state in the decoder LSTM as well as the size of the embedding of $\mathcal{R}(n_j, i_j)$ to be 100. To obtain the contexts, we use the Luong general attention scheme~\citep{attention}. To prevent overfitting, we apply a dropout of 0.3 to all the embedding, outputs of recurrent modules, and context vectors. Our model is trained using Adam~\citep{adam} with a learning rate of 0.003 and a batch size of 25.

\section{\textsc{SelectLeaf Function} Details}
The \textsc{SelectLeaf} function selects one non-terminal leaf node in the partial program to expand. We find that when programmatic constraints are integrated into the search process, the order of choose which non-terminal to expand can impact the cost needed to synthesize the target program. We give a concrete example of how the way we select non-terminal leaf nodes to expand can affect the cost of synthesis. 
Consider a timestep where we obtain the feasible partial program \texttt{\small cat($V_1$,$V_2$)} from the queue, where both $V_1$ and $V_2$ can be expanded to \regdsl{<0>} or \regdsl{<1>} with a probabilities $0.9$ and $0.1$ respectively. Suppose \regdsl{cat(<0>,$V_2$)} is feasible, \regdsl{cat($V_1$,<0>)} is infeasible, and the only feasible complete program is \texttt{\small cat(<1>,<1>)}. If we choose to expand $V_1$ first, then the search procedure goes as follows: \{(\regdsl{cat(<0>,$V_2$)}, \cmark) $\rightarrow$ (\regdsl{cat(<0>,<0>)},\xmark) $\rightarrow$ (\regdsl{cat(<0>,<1>)},\xmark)$\rightarrow$ (\regdsl{cat(<1>,$V_2$)},\xmark)$\rightarrow$ (\regdsl{cat(<1>,<0>)},\xmark)$\rightarrow$ (\regdsl{cat(<1>,<1>)},\cmark)\}, which takes 6 steps. Now, if we expand $V_2$ first, the search procedure is: \{(\regdsl{cat($V_1$,<0>)}, \xmark) $\rightarrow$ (\regdsl{cat($V_1$,<1>)},\cmark),$\rightarrow$ (\regdsl{cat(<0>,<1>)},\xmark),$\rightarrow$ (\regdsl{cat(<1>,<1>)},\cmark)\}, which only takes 4 steps.

We want to find an order to expand the nodes that leads to most effective pruning. We tested the following ways of selecting leaf nodes: (1) pre-order traversal, (2) choosing the highest-level leaf node, (3) choosing the lowest-entropy leaf node.
We found that pre-order traversal worked better than the other strategies in most cases.
Given the same budget, using per-order traversal solves more programs while exploring fewer states compared to the other ways.
The superiority of pre-order traversal on the regex synthesis task can be attributed to that our {\sc Infeasible} function needs concrete terminal leaf nodes to prune effectively, and using pre-order traversal prioritizes deepest nodes and usually yields terminal leaf nodes more quickly than other strategies.
\section{Implementation Details of the Baselines}
\paragraph{\textsc{AlphaRegex}} We implemented the top-down enumerative synthesizer presented in \cite{lee}. Although \cite{lee} only uses \texttt{<0>} and \texttt{<1>} as terminals, here we extended the synthesizer to support most of the ASCII characters. 

\paragraph{\textsc{DeepCoder}}
We implemented \textsc{DeepCoder} with a few modifications from its original implementation \citep{deepcoder}.
First, we assign each token in the examples with a \emph{class}, and embed the token by both its value and its class. For instance, consider the positive example \texttt{\small (ax4,+)} of the regex \texttt{\small concat(repeat(<low>,2),repatleast(<num>, 1)} (2 lower letters followed by 1 or more digits. We assign ``\texttt{a}'' and  ``\texttt{b}'' with the ``\texttt{<low>}'' class, and assign ``\texttt{4}'' with the ``\texttt{<num>}'' class. The final embedding of the token ``\texttt{a}'' is the concatenation of the embedding of the value $\text{Emb}(\texttt{\small a})$ and the class $\text{Emb}({\texttt{<low>}})$. We use such combined embeddings for better generalizability. Then, we encode the examples with a Bi-LSTM encoder. Each example is encoded into a hidden vector, which is later max-pooled. Finally, we apply a linear layer on the pooled representation for the whole program, and predict the the set of probabilities for each of the constructs in the DSL.

We extended {\sc AlphaRegex} to synthesize programs using the probability of constructs obtained from the neural model. In the \textsc{StructuredRegex} grammar, we associate each construct with the score returned from the neural network and calculate the score of a partial program by summing up the score of all the constructs that are used in the partial program. We specify the synthesizer to prioritize exploring the partial programs with the highest score so far. 

Recall that in Section~\ref{sec:results} that {\sc DeepCoder} doesn't achieve high performance in the {\sc StructureRegex} dataset. Since  most of the constructs are recursive in the regex language and {\sc DeepCoder} search is essentially doing a depth-first search, the synthesizer first needs to exhaustively check all possible programs associated with the highest probability constructs before it can move on to explore those programs with any other constructs. For example, suppose the {\tt concat} has the highest probability and the synthesizer explores programs up to maximum depth $5$, the synthesizer will prioritize exploring programs like {\tt \small concat(concat(concat(concat(<low>))))} and searching in this way does not help the synthesizer to find the ground truth regex. 

\paragraph{\textsc{RobustFill}}
We implemented the {\sc Attention A} model from {\sc RobustFill} \citep{robustfill}, which predicts programs given I/O examples. We encode the I/O with the same I/O embedding and I/O encoder used in our implementation of {\sc DeepCoder}. We replaced the LSTM decoder in the original implementation with our ASN decoder. During decoding, we extract a context vector from each of the examples provided in the example set, and pool them with max-pooling as the final context vector. The probability distribution over rules for node $n$ is then given as:
\begin{align*}
    \text{Attn}(h_n, \text{context}(\spec))= \text{MaxPool}(\\ \{\text{Attn}(h_n,\text{context}(e))\}_{e\in\examples})
\end{align*}

\begin{align*}
    p_\theta(r|n,P,N)=\text{softmax}(\text{FFNN}(h_n; \\ \text{Attn}(h_n, \text{context}(\spec))))
\end{align*}

We set the size of value embedding and class embedding to be 50, and the size of hidden state in encoder Bi-LSTM and LSTM in ASN to be 100. 

\paragraph{\textsc{TreeSearch}} As the code of {\sc TreeSearch} \citep{algolisp} is not publicly available code, we implemented our own version of {\sc TreeSearch} on top of {\sc TranX} which is reported to be more powerful than the originally used {\sc Seq2Tree} on various datasets \citep{tranxdemo}. During search, we set the threshold to be $10^{-5}$, and the max queue size to be 100.

\paragraph{\textsc{OpSynth\textsuperscript{\sc +$\mathcal{R}$}}}
We naturally combine {\sc OpSynth} and {\sc RobustFill} by concatenating the context vectors from NL and examples, as in Section~\ref{sec:implementation}. The hyper-parameters for the NL encoder are the same as those for the base synthesizer, and the hyper-parameters for the I/O encoder are the same as {\sc RobustFill}.

\end{document}